\documentclass[10pt,twocolumn,letterpaper]{article}
\usepackage[pagenumbers]{cvpr}
\usepackage[T1]{fontenc}
\usepackage{mathtools}
\usepackage{amsthm}
\usepackage{bm}
\usepackage{algorithm}
\usepackage{algorithmic}
\usepackage{array}
\usepackage{multirow}
\usepackage{makecell}
\usepackage{colortbl}
\usepackage{textcomp}
\usepackage{placeins}
\definecolor{gbblue}{HTML}{315A7D}
\definecolor{gbteal}{HTML}{4F8B88}
\definecolor{gborange}{HTML}{D57A4A}
\definecolor{gblight}{HTML}{EEF3F6}
\definecolor{gbrowblue}{HTML}{EAF3FC}
\definecolor{gbwarm}{HTML}{FBF1EA}
\definecolor{gbgray}{HTML}{F3F4F5}
\definecolor{gbink}{HTML}{263238}
\newcommand{\best}[1]{\textbf{#1}}
\newcommand{\second}[1]{\underline{#1}}
\theoremstyle{plain}
\newtheorem{proposition}{Proposition}
\theoremstyle{definition}

\theoremstyle{remark}

\definecolor{cvprblue}{rgb}{0.21,0.49,0.74}
\usepackage[breaklinks,colorlinks,allcolors=cvprblue]{hyperref}
\hypersetup{pdftitle={GeoBalance: Geometry-Aware Monitoring and Reconstruction with Asymmetric Optimization for Balanced Multimodal Learning},pdfauthor={Zechang Xiong, Da Li, Rong Yin, Kexin Tang, Biao Yang, Pengyuan Li, Wenkang Kong, Yulan Hu, Shengyu Zhu, Hao Peng},pdfsubject={Multimodal learning}}
\title{GeoBalance: Geometry-Aware Monitoring and Reconstruction with Asymmetric Optimization for Balanced Multimodal Learning}
\newcommand{\preprintauthors}{%
\begin{minipage}{\textwidth}
\centering
{\fontsize{11}{14}\selectfont
\setlength{\tabcolsep}{0pt}%
\begin{tabular*}{0.90\textwidth}{@{\extracolsep{\fill}}ccccc@{}}
Zechang Xiong$^{1,*}$ & Da Li$^{2,*}$ & Rong Yin$^{1,\dagger}$ & Kexin Tang$^{3}$ & Biao Yang$^{4}$\\[3pt]
Pengyuan Li$^{5}$ & Wenkang Kong$^{5}$ & Yulan Hu$^{6}$ & Shengyu Zhu$^{2}$ & Hao Peng$^{1}$
\end{tabular*}\par}
\vspace{9pt}
{\fontsize{9}{11.5}\selectfont
\setlength{\tabcolsep}{0pt}%
\begin{tabular}{@{}c@{}}
$^{1}$Beihang University \qquad $^{2}$Institute of Computing Technology, Chinese Academy of Sciences\\[2pt]
$^{3}$Amap, Alibaba Group \quad $^{4}$Kuaishou Technology \quad $^{5}$Beijing Jiaotong University \quad $^{6}$Renmin University of China
\end{tabular}\par}
\end{minipage}%
}
\author{\preprintauthors}
\makeatletter
\renewcommand{\@maketitle}{%
  \newpage\null\vspace{10pt}%
  \begin{center}
    {\Large\bfseries\@title\par}
    \vspace{14pt}
    \@author\par
    \vspace{8pt}
  \end{center}%
}

\makeatother
\begin{document}
\raggedbottom
\maketitle
\begingroup
\renewcommand{\thefootnote}{\fnsymbol{footnote}}
\footnotetext[1]{Equal contribution.}
\footnotetext[2]{Corresponding author: \href{mailto:yinrong@buaa.edu.cn}{\textcolor{black}{yinrong@buaa.edu.cn}}.}
\endgroup
\begin{abstract}
Multimodal classifiers can converge to modality-dominant solutions in which one modality dominates the joint prediction, suppressing the learning of others. Existing balancing methods mainly adjust losses, gradients, or modality contributions, largely treating modality imbalance as an optimization problem while implicitly treating the weak modality as under-optimized but representationally intact. In this work, we find that this assumption does not always hold, as persistent modality dominance can induce a representation-level collapse of the weak modality, which we term \emph{manifold modality collapse} (MMC). MMC manifests as a coupled geometric degradation in which weak-modality representations collapse onto fewer directions within each class and become less separable across classes. Motivated by this observation, we propose \emph{GeoBalance}, a geometry-aware framework that monitors these two geometric properties and reconstructs the weak modality representation only when it exhibits signs of MMC. Once triggered, GeoBalance uses a fixed Simplex-ETF class scaffold and spectral regularization to restore class separation while preventing collapse onto a few feature directions. To preserve reconstruction during joint training, asymmetric gradient projection removes the joint-gradient component conflicting with reconstruction, leaving non-conflicting optimization unchanged. Extensive experiments across six multimodal benchmarks demonstrate great improvements over competitive balancing methods, validating its effectiveness.
\end{abstract}

\section{Introduction}
\label{sec:intro}
Multimodal learning exploits complementary evidence to resolve ambiguities that individual modalities cannot reliably address~\cite{xue2023multilevel,zhang2024lowqualitysurvey,jin2026taskaware}. Yet fusion alone does not ensure complementarity: each branch must retain information useful to joint prediction~\cite{wang2025mmcn}. Differences in signal quality, task relevance, and optimization difficulty can allow one modality to dominate joint learning and suppress weaker branches~\cite{wang2020what,peng2022balanced,huang2022modality,liu2024uekd}. Such \emph{modality dominance} may be hidden by competitive multimodal accuracy, leaving the predictor brittle when its dominant input is corrupted or missing~\cite{lee2023multimodal,ma2025dynamic}.

\begin{figure}[t]
  \centering
  \includegraphics[width=0.9\linewidth]{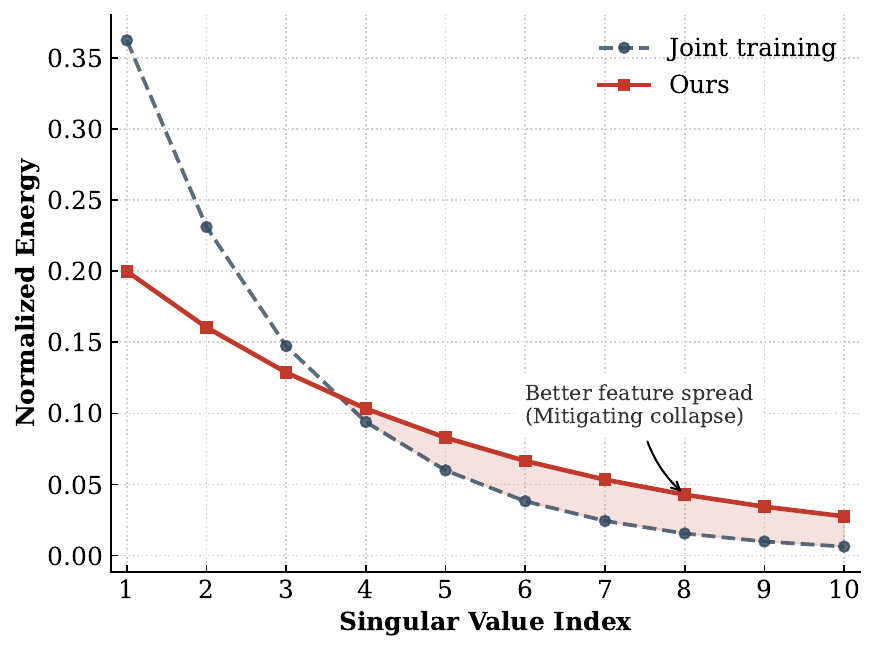}
  \caption{Normalized singular-value energy of weak-modality representations. GeoBalance yields a broader spectrum than joint training. MMC additionally requires degraded class separability.}
  \label{fig:intro}
\end{figure}

Most existing balancing methods approach modality dominance as an optimization-allocation problem. They modulate gradients~\cite{peng2022balanced,li2023agm}, reshape losses or classifier interactions~\cite{xu2023mmcosine,fan2023pmr}, alternate modality-specific learning~\cite{zhang2024mla}, or coordinate objectives~\cite{wei2024mmpareto}. This optimization-centered perspective primarily addresses differences in learning progress, while the geometric state of the weak modality often remains implicit. Under persistent dominance, however, the challenge may extend beyond insufficient optimization to deterioration of the representation itself.

In this work, we identify such a representation-level failure mode. As illustrated in \cref{fig:intro}, weak-modality feature energy can concentrate along a few directions under joint training. Such concentration is not inherently harmful, since useful information may occupy a compact manifold. The failure becomes consequential when reduced within-class effective dimensionality is accompanied by degraded class separability. This coupled degradation, termed \emph{manifold modality collapse} (MMC), describes a loss of usable class-conditional geometry.

MMC calls for intervention that follows the weak modality's evolving geometric state. A weak branch may preserve useful geometry early in training yet deteriorate later. Conversely, imposing reconstruction throughout training may disrupt useful co-adaptation. The goal is therefore to distinguish a weak but usable representation from one requiring reconstruction, without requiring equal modality strengths.

Guided by this view, we propose \emph{GeoBalance}, a geometry-aware framework organized as a monitor-reconstruct-project process. The monitor activates reconstruction when both geometric properties fall below early-training reference thresholds. A fixed Simplex equiangular tight frame (ETF) guides class separation, while spectral regularization discourages concentration in the auxiliary space. When joint optimization conflicts with reconstruction, asymmetric projection corrects the opposing joint-gradient component and leaves the reconstruction gradient intact. The resulting intervention is selective in training time and gradient direction, with all auxiliary components removed at inference.

The main contributions are threefold:
\begin{itemize}
  \item \textbf{Geometric failure formulation.} MMC characterizes the coupled degradation of weak-modality effective dimensionality and class separability, providing a representation-level criterion for intervention.
  \item \textbf{State-conditioned reconstruction.} GeoBalance combines geometric reconstruction with asymmetric gradient protection, targeting deteriorated representations without requiring equal modality strength.
  \item \textbf{Empirical evaluation.} Experiments across six multimodal benchmarks show improved average predictive performance and examine representation geometry, component effects, and monitoring behavior.
\end{itemize}

\begin{figure*}[t]
  \centering
  \includegraphics[width=\textwidth]{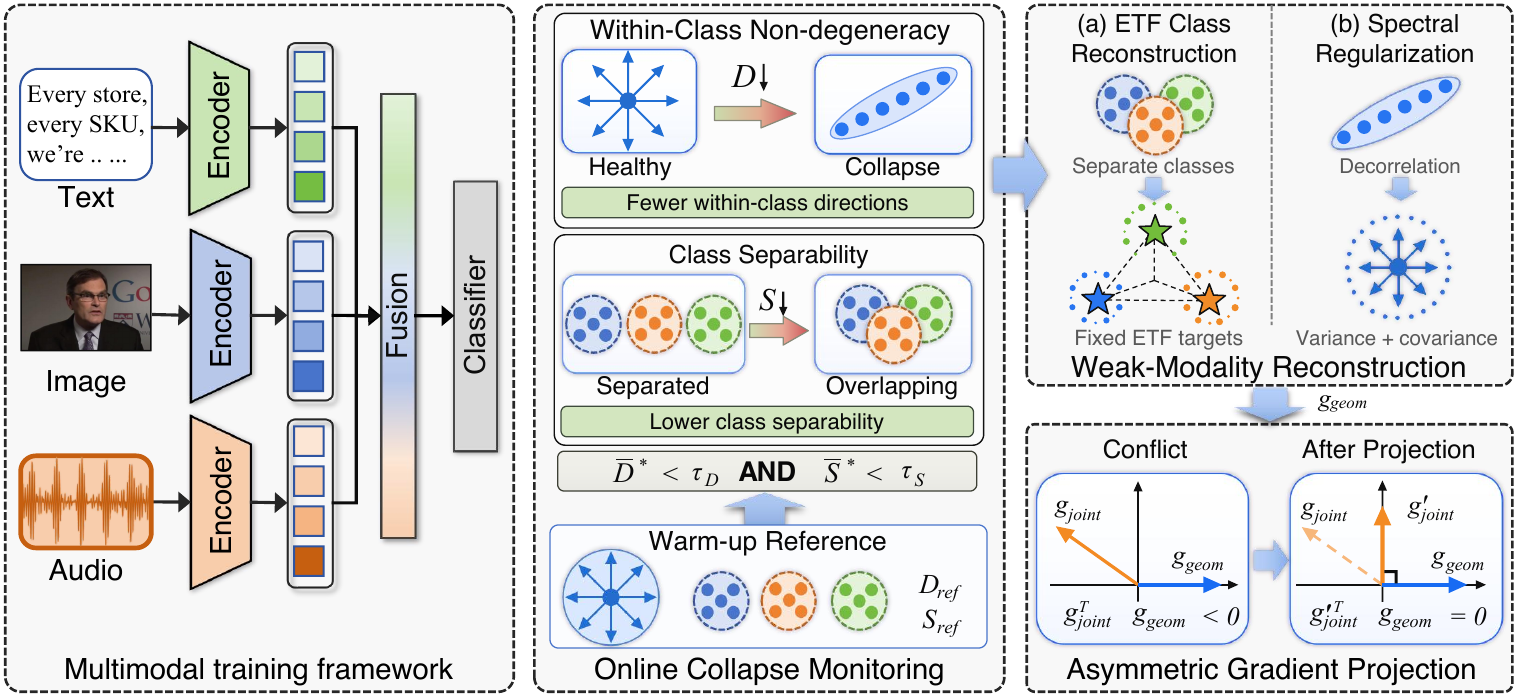}
  \caption{GeoBalance overview. Monitoring within-class dimension and class separability gates ETF reconstruction and spectral regularization; asymmetric correction limits conflicting joint updates. Auxiliary components are removed at inference.}
  \label{fig:method}
\end{figure*}
\section{Related Work}\label{sec:related}
\subsection{Multimodal Representation Learning}
Multimodal representation learning combines shared semantics with modality-specific evidence. Cross-modal reconstruction~\cite{ngiam2011multimodal} and attention~\cite{tsai2019mult} capture interactions, while contrastive pretraining aligns modalities in transferable shared spaces~\cite{radford2021clip,girdhar2023imagebind}. Task-oriented fusion~\cite{xue2023multilevel,mu2024mocolnet,jin2026taskaware}, shared--private factorization~\cite{hazarika2020misa}, and alternating adaptation~\cite{zhang2024mla} further improve representation quality. These approaches emphasize correspondence and predictive utility; GeoBalance examines whether a weak branch retains usable class-conditional geometry during joint training.

\subsection{Balanced Multimodal Learning}
Unequal learning dynamics and competition through fusion motivate several complementary approaches to modality imbalance~\cite{wang2020what,huang2022modality,wu2022characterizing}.

\paragraph{Optimization modulation and coordination.}
Loss weighting and gradient modulation regulate relative learning progress using generalization behavior or modality-specific learning status to determine update strength~\cite{wang2020what,peng2022balanced,li2023agm}. Multi-objective coordination~\cite{wei2024mmpareto}, encoder--fusion gradient disentanglement~\cite{wei2025dgl}, and joint modulation of gradient magnitudes and directions~\cite{gao2026reconcile} address interference at different levels. Conflict projection is also established in multi-task learning~\cite{yu2020gradient}. BALM extends feature calibration and gradient rebalancing to unequal modality missing rates~\cite{nguyen2026balm}.

\paragraph{Unimodal representation enhancement.}
Classifier normalization and prototype supervision strengthen branch-specific discrimination~\cite{xu2023mmcosine,fan2023pmr,chen2026psl}. Alternating adaptation, relearning, and boosting provide additional or more independent unimodal optimization~\cite{zhang2024mla,wei2024diagnosing,hua2024reconboost}. Data Remixing selects hard examples for each modality and reassembles them at the batch level to improve learning balance and reduce interference~\cite{ma2025remixing}. Representation-level analyses further show how entanglement between noisy and predictive features in the fusion head can suppress useful modality contributions, motivating explicit basis reallocation to relieve rank bottlenecks~\cite{chaudhuri2025collapse}.

\paragraph{Information and contribution modeling.}
Sample-level valuation and causal analysis estimate heterogeneous modality contributions~\cite{wei2024valuation,wang2025cmob}. Information-based methods regulate early information acquisition~\cite{huang2025inforeg}, preserve complementary information during fusion~\cite{qin2026deep}, or maximize total correlation between features and labels~\cite{yu2026tcmax}. Intrinsic information budgets and performance-dominant modality prioritization further question whether equal modality strength is desirable~\cite{xiong2026forcedmodalitybalanceintrinsic,wei2026pdmp}.

GeoBalance focuses on the evolving class-conditional geometry of the weak branch. Its criterion couples reduced within-class effective dimensionality with degraded class separability, distinguishing MMC from rank deficiency alone. This state determines when reconstruction is activated, while asymmetric projection protects reconstruction from conflicting joint updates.

\subsection{Representation Geometry and Neural Collapse}
Neural collapse couples within-class contraction with well-separated Simplex ETF class means and aligned classifiers~\cite{papyan2020prevalence}. Its geometry has been analyzed through unconstrained-feature models~\cite{DBLP:conf/nips/ZhuDZLYSQ21} and exploited in imbalanced and transfer learning~\cite{DBLP:conf/nips/YangCLXLT22,li2022understanding}. Such contraction need not indicate harmful collapse: MMC additionally requires degraded class separability. GeoBalance uses the ETF as a class-level reconstruction scaffold, together with a variance--covariance surrogate that discourages directional concentration in the auxiliary space.

\section{Geometric Formulation}
\label{sec:preliminaries}
\paragraph{Problem setup.}
Consider $N$ labeled samples with $M\ge2$ modalities and $K\ge2$ classes.
For sample $i$, $x_i^{(m)}$ is the input from modality $m$ and
$y_i\in\{1,\ldots,K\}$ is its label. Encoder $f_m$ produces
$\mathbf z_i^{(m)}=f_m(x_i^{(m)})\in\mathbb R^{d_m}$, where $d_m$ is the
feature dimension. A fusion module $\phi$ and classifier $h$ are trained
on minibatch index set $\mathcal B_t$ at iteration $t$:
\begin{equation}
\mathcal L_{\mathrm{joint}}
=\frac1B\sum_{i\in\mathcal B_t}
\mathrm{CE}\!\left(h\!\left(\phi(\{\mathbf z_i^{(m)}\}_{m=1}^M)\right),y_i\right).
\label{eq:joint_objective}
\end{equation}
Here $B=|\mathcal B_t|\ge2$, and $\mathrm{CE}$ is cross entropy on class
logits. Iteration indices are omitted from features and losses.
After separate unimodal training under the same setting, the branch with
the lowest training performance is designated $m_\star$.
Superscript $\star$ denotes this fixed branch, with $d_\star=d_{m_\star}$.

\paragraph{Class-conditional geometry.}
At a fixed model state, let $\mathbf z^{(m)}$ denote a random
modality-$m$ representation and $y$ its label.
The within- and between-class scatter matrices are
\begin{equation}
\begin{aligned}
\boldsymbol\Sigma_{\mathrm w}^{(m)}
 &=\mathbb E_y[\operatorname{Cov}(\mathbf z^{(m)}\mid y)],\\
\boldsymbol\Sigma_{\mathrm b}^{(m)}
 &=\operatorname{Cov}_y(\mathbb E[\mathbf z^{(m)}\mid y]),
\end{aligned}
\label{eq:population_scatter}
\end{equation}
where expectation and covariance over $y$ use the class prior.
Within-class effective dimensionality is
\mbox{$D^{(m)}=d_{\mathrm{eff}}(\boldsymbol\Sigma_{\mathrm w}^{(m)})$}.
For positive semidefinite $\mathbf A$, define
\begin{equation}
d_{\mathrm{eff}}(\mathbf A)
=\frac{\operatorname{tr}(\mathbf A)^2}{\|\mathbf A\|_F^2+\epsilon}.
\label{eq:effective_dimension}
\end{equation}
Here $\operatorname{tr}$, $\|\cdot\|_F$, and $\|\cdot\|_2$ denote trace,
Frobenius norm, and Euclidean norm;
$\epsilon>0$ is a numerical stabilizer.
Without this stabilizer, the ratio is scale invariant for nonzero
$\mathbf A$ and equals $r$ for $r$ equal nonzero eigenvalues.
In this ideal case, uniform contraction does not change the measure.
Class separability is the complementary trace ratio
\begin{equation}
\mathcal S^{(m)}
=\frac{\operatorname{tr}(\boldsymbol\Sigma_{\mathrm b}^{(m)})}
       {\operatorname{tr}(\boldsymbol\Sigma_{\mathrm w}^{(m)})+\epsilon}.
\label{eq:class_separability}
\end{equation}
MMC denotes a joint decrease in $D^{(m_\star)}$ and
$\mathcal S^{(m_\star)}$ relative to an earlier state of the same branch.
Neither statistic alone suffices: compact features can remain
discriminative, while broad feature variation need not separate classes.

\paragraph{Class reconstruction scaffold.}
A fixed Simplex ETF provides class targets independent of the dominant
modality. Its prototype matrix
$\mathbf C=[\mathbf c_1,\ldots,\mathbf c_K]\in\mathbb R^{d_e\times K}$
satisfies
\begin{equation}
\mathbf C^\top\mathbf C
=\frac K{K-1}\left(\mathbf I_K-\frac1K\mathbf1_K\mathbf1_K^\top\right).
\label{eq:etf_scaffold}
\end{equation}
Here $d_e\ge K-1$ is the auxiliary dimension, $\mathbf I_K$ is the identity
matrix, and $\mathbf1_K$ is the all-ones vector.
The columns have unit norm and pairwise inner product $-1/(K-1)$,
giving equally separated targets for all classes.

\section{Method}
\label{sec:method}
\label{sec:method_overview}
GeoBalance follows the monitor--reconstruct--project process in
\cref{fig:method}. The monitor determines when weak-modality geometry has
deteriorated; reconstruction supplies a class target and an anti-collapse
penalty; asymmetric projection limits interference from the joint objective.
These roles distinguish selecting an intervention from specifying and
protecting its optimization direction.

\subsection{Online Collapse Monitoring}
\label{sec:online_monitor}
A weak branch may still develop useful class structure early in training.
Applying reconstruction throughout training could restrict this
co-adaptation. The monitor therefore detects deterioration relative to
the branch's own early state. Its detached statistics do not participate
in backpropagation.

Let $\mathcal Y_t$ contain the labels observed in $\mathcal B_t$.
For $k\in\mathcal Y_t$, define the class index set
$\mathcal B_{t,k}=\{i\in\mathcal B_t:y_i=k\}$ and its size $n_{k,t}$.
Let $\widehat{\boldsymbol\mu}_{k,t}^\star$ and
$\widehat{\boldsymbol\mu}_t^\star$ be the class and overall minibatch
feature means. Stack the class-centered row vectors
$(\mathbf z_i^\star-\widehat{\boldsymbol\mu}_{y_i,t}^\star)^\top$
into $\widetilde{\mathbf Z}_t^\star\in\mathbb R^{B\times d_\star}$.
The raw Gram matrix
$\mathbf G_t^\star=\widetilde{\mathbf Z}_t^\star\widetilde{\mathbf Z}_t^{\star\top}$
gives the dimension estimate
\begin{equation}
\widehat D_t^\star=d_{\mathrm{eff}}(\mathbf G_t^\star).
\label{eq:batch_dimension}
\end{equation}
This uses the nonzero spectrum of the unnormalized feature scatter and
requires no eigendecomposition.
The corresponding class-separability estimate is
\begin{equation}
\widehat{\mathcal S}_t^\star
=\frac{\sum_{k\in\mathcal Y_t}n_{k,t}
 \|\widehat{\boldsymbol\mu}_{k,t}^\star-\widehat{\boldsymbol\mu}_t^\star\|_2^2}
 {\|\widetilde{\mathbf Z}_t^\star\|_F^2+B\epsilon}.
\label{eq:batch_separability}
\end{equation}
The class counts preserve empirical proportions.
Because $\mathbf G_t^\star$ is unnormalized, its effective dimension
equals that of the empirical within-class covariance with stabilizer
$\epsilon/B^2$.
Class centering also limits its rank to
$\min(d_\star,B-|\mathcal Y_t|)$, and singleton classes contribute no
within-class variation. These estimates therefore depend on batch
composition even after smoothing.

For $a\in\{D,\mathcal S\}$, the exponential moving average (EMA)
and warm-up reference are defined by
\begin{align}
\overline a_t^\star
 &=\rho\overline a_{t-1}^\star+(1-\rho)\widehat a_t^\star,
 \label{eq:ema_monitor}\\
a^{\mathrm{ref}}
 &=\operatorname{median}_{1\le s\le T_0}\overline a_s^\star.
 \label{eq:warmup_reference}
\end{align}
The smoothing factor is $\rho\in[0,1)$; each EMA starts from its first
estimate. During $T_0\ge1$ warm-up iterations, only the joint objective is
used. References are then fixed, while EMA updates continue.
With relative threshold $\eta\in(0,1)$, set
$\tau_a=\eta a^{\mathrm{ref}}$ for each statistic.
After warm-up, the gate is
\begin{equation}
\mathbb I_t=\mathbf1\!\left[
 \overline D_t^\star<\tau_D
 \ \wedge\
 \overline{\mathcal S}_t^\star<\tau_{\mathcal S}
 \right].
\label{eq:collapse_gate}
\end{equation}
Here $\mathbf1[\cdot]$ is the indicator function, and
$\mathbb I_t=0$ during warm-up. Requiring both decreases avoids treating
dimensional compression or poor separability alone as MMC.
Larger $\eta$ requires less deterioration.
The same gate activates both reconstruction terms and gradient correction,
and can turn off again as the monitored state changes.
Branch-relative references avoid requiring geometric statistics to match
across modalities. A weak but stable branch can remain ungated, while
fixing the reference prevents sustained deterioration from redefining
the baseline. Detection is relative and requires nondegenerate warm-up
statistics.

\subsection{Weak-Modality Reconstruction}
\label{sec:weak_reconstruction}
Detection alone provides no target for restoring discriminative structure.
GeoBalance combines a fixed class scaffold and a penalty against
concentrated auxiliary features, without requiring imitation of the
dominant modality.

\paragraph{ETF class reconstruction.}
A training-only head $g_\star$ maps the weak feature to
$\mathbf u_i^\star=g_\star(\mathbf z_i^\star)\in\mathbb R^{d_e}$,
where $d_e=\max(d_\star,K-1)$, with $d_e\ge2$.
Using the prototypes in \cref{eq:etf_scaffold}, define
\begin{equation}
\mathcal L_{\mathrm{ETF}}^\star
=\frac1B\sum_{i\in\mathcal B_t}
\mathrm{CE}\!\left(
s_\star\mathbf C^\top
\frac{\mathbf u_i^\star}{\|\mathbf u_i^\star\|_2+\epsilon},y_i
\right).
\label{eq:etf_loss}
\end{equation}
The positive learnable scale $s_\star$ is initialized to $10$.
Fixed, equally separated class directions supply a stable target, while
normalization emphasizes feature direction.
The auxiliary head accommodates the ETF dimension without changing the
representation dimension used by fusion.
This class target remains available even when joint prediction relies
mainly on another modality, allowing the weak branch to develop its own
class-relevant cues.

\paragraph{Spectral regularization.}
Class supervision alone does not prevent feature variation from relying
on a few directions. Let $\boldsymbol\mu_{u,t}^\star$ be the mean projected
feature and let $\mathbf U_{c,t}^\star\in\mathbb R^{B\times d_e}$
stack rows $(\mathbf u_i^\star-\boldsymbol\mu_{u,t}^\star)^\top$.
Its sample covariance is
\begin{equation}
\mathbf Q_t^\star
=\frac{\mathbf U_{c,t}^{\star\top}\mathbf U_{c,t}^\star}{B-1}.
\label{eq:projected_covariance}
\end{equation}
Write $Q_{t,pq}^\star$ for entry $(p,q)$, where $p,q\in\{1,\ldots,d_e\}$,
and $\sigma_{j,t}=\sqrt{Q_{t,jj}^\star}$ for the standard deviation of
coordinate $j$.
Variance--covariance regularization~\cite{bardes2022vicreg} gives
\begin{equation}
\begin{split}
\mathcal L_{\mathrm{spec}}^\star
={}&\frac1{d_e}\sum_{j=1}^{d_e}\max(0,\gamma-\sigma_{j,t})\\
&+\frac1{d_e(d_e-1)}\sum_{p\ne q}(Q_{t,pq}^\star)^2,
\end{split}
\label{eq:spectral_loss}
\end{equation}
with standard-deviation floor $\gamma=1$.
The first term discourages inactive coordinates; the second penalizes
off-diagonal covariance. Together they provide an anti-collapse surrogate
in the auxiliary space, not a direct optimization of the class-centered
dimension monitored in \cref{eq:batch_dimension}.
The reconstruction objective is
\begin{equation}
\mathcal L_{\mathrm{geo}}^\star
=\mathcal L_{\mathrm{ETF}}^\star+\mathcal L_{\mathrm{spec}}^\star.
\label{eq:geometry_loss}
\end{equation}
After batch averaging and dimension normalization, both terms use unit
coefficients.

\subsection{Asymmetric Gradient Projection}
\label{sec:gradient_projection}
Even when reconstruction is activated, the fused objective can favor an
update opposing it. Modifying both gradients symmetrically could also
weaken the auxiliary reconstruction direction.
GeoBalance therefore corrects the joint gradient alone.

Let $\theta_{\mathrm{sh}}$ be the trainable parameters on which both losses
depend, such as the weak encoder or a shared input adapter.
\emph{Shared} refers to this dependency, not weight sharing between modalities
or the currently nonzero gradient coordinates. Define
\begin{equation}
\begin{aligned}
\mathbf g_J&=\nabla_{\theta_{\mathrm{sh}}}\mathcal L_{\mathrm{joint}},\\
\mathbf g_G&=\nabla_{\theta_{\mathrm{sh}}}\mathcal L_{\mathrm{geo}}^\star.
\end{aligned}
\label{eq:two_gradients}
\end{equation}
Their inner product $\delta_t=\langle\mathbf g_J,\mathbf g_G\rangle$
is negative under first-order conflict. The gated correction is
\begin{equation}
\Pi_t(\mathbf g_J;\mathbf g_G)
=\mathbf g_J-\mathbb I_t
 \frac{\min\{\delta_t,0\}}{\|\mathbf g_G\|_2^2+\epsilon}\mathbf g_G.
\label{eq:projection_operator}
\end{equation}
The optimizer receives
\begin{equation}
\mathbf g_{\mathrm{sh}}
=\Pi_t(\mathbf g_J;\mathbf g_G)+\mathbb I_t\mathbf g_G.
\label{eq:shared_gradient}
\end{equation}
An inactive gate gives the ordinary joint gradient.
On an active step, agreeing gradients are combined unchanged; conflicting
gradients trigger a correction along $\mathbf g_G$, leaving the
reconstruction gradient intact.
Thus the geometric gate decides whether to intervene, while the conflict
test decides whether the joint gradient needs correction.

For an active gate, $\epsilon=0$, and nonzero $\mathbf g_G$, the correction is the
minimum-Euclidean-change projection onto the half-space having nonnegative
inner product with $\mathbf g_G$.
The stabilized formula attenuates conflict but is not exactly
orthogonal when $\epsilon>0$; it is the identity when $\mathbf g_G=0$.
For a conflicting step, the ideal correction removes only the component
of $\mathbf g_J$ opposing $\mathbf g_G$ and leaves its orthogonal component
unchanged. With stabilization, their inner product becomes
$\epsilon\delta_t/(\|\mathbf g_G\|_2^2+\epsilon)$, making the residual
conflict explicit.
This is a local gradient property, not a guarantee of monotonic improvement
under momentum, adaptive preconditioning, or finite steps.

\subsection{Training and Inference}
\label{sec:optimization}
The losses specify the supervision, while \cref{eq:shared_gradient}
specifies the modified update on $\theta_{\mathrm{sh}}$.
Joint-only parameters receive the joint gradient; geometry-only parameters
receive the gated geometry gradient.
The correction is applied before updating optimizer moments.
When $\mathbb I_t=0$, reconstruction and its additional gradient are skipped,
but detached monitoring continues.
Each iteration first computes joint prediction and detached statistics.
Warm-up establishes the references; later iterations update the EMAs and
test the gate. An active gate adds reconstruction, tests gradient conflict,
and corrects the shared update before the optimizer step.

Reconstruction changes the monitored representation through the weak
encoder, which is optimized together with the other modality branches.
The auxiliary head provides a differentiable path from the geometric
objectives to the features consumed by fusion.
Only parameters used by both objectives receive the corrected shared
update: other modality encoders and the fusion classifier retain their
joint-objective gradients, while the auxiliary head receives the geometry
gradient on active steps.
Thus end-to-end training does not require applying gradient projection
to every model parameter.
At inference, the auxiliary head, ETF prototypes, monitor, and gate states
are removed, leaving the deployed architecture unchanged.

\section{Experiments}
\label{sec:experiments}

\begin{table*}[t]
\caption{Comparison on six multimodal benchmarks (\%). Bold and underline mark the best and second-best reported results. Averages cover all six datasets and are omitted for incomplete rows.}
\label{tab:main_results}
\centering
\footnotesize
\setlength{\tabcolsep}{4.7pt}
\renewcommand{\arraystretch}{1.0}
\begin{tabular*}{\textwidth}{@{\extracolsep{\fill}}l*{14}{c}@{}}
\toprule
\multirow{2}{*}{Method} & \multicolumn{2}{c}{Kinetics-Sounds} & \multicolumn{2}{c}{CREMA-D} & \multicolumn{2}{c}{VGGSound} & \multicolumn{2}{c}{CMU-MOSEI} & \multicolumn{2}{c}{UCF101} & \multicolumn{2}{c}{Food-101} & \multicolumn{2}{c}{Average} \\
\cmidrule(lr){2-3}\cmidrule(lr){4-5}\cmidrule(lr){6-7}\cmidrule(lr){8-9}\cmidrule(lr){10-11}\cmidrule(lr){12-13}\cmidrule(lr){14-15}
& Acc. & $F_1$ & Acc. & $F_1$ & Acc. & $F_1$ & Acc. & $F_1$ & Acc. & $F_1$ & Acc. & $F_1$ & Acc. & $F_1$ \\
\midrule
Joint training & 65.42 & 65.23 & 65.53 & 65.11 & 47.79 & 46.77 & 79.02 & 70.12 & 80.92 & 80.22 & 91.33 & 91.34 & 71.67 & 69.80 \\
G-Blending~\cite{wang2020what} & 68.47 & 66.32 & 71.53 & 71.52 & \second{51.41} & \second{50.36} & 79.63 & 73.27 & \second{85.12} & \second{84.77} & 92.72 & 92.66 & 74.81 & 73.15 \\
Greedy~\cite{wu2022characterizing} & 66.77 & 66.42 & 66.43 & 55.51 & 48.62 & 47.42 & -- & -- & -- & -- & -- & -- & -- & -- \\
OGM-GE~\cite{peng2022balanced} & 67.08 & 66.98 & 67.53 & 67.93 & 50.32 & 48.73 & 80.43 & 73.55 & 82.11 & 81.41 & 91.55 & 91.53 & 73.17 & 71.69 \\
MMCosine~\cite{xu2023mmcosine} & 67.49 & 67.09 & 67.19 & 67.34 & 48.73 & 47.66 & 80.38 & 73.67 & 82.97 & 82.47 & 92.16 & 92.12 & 73.15 & 71.73 \\
PMR~\cite{fan2023pmr} & 67.16 & 66.89 & 67.29 & 67.23 & 50.48 & 49.22 & 79.79 & 72.06 & 81.73 & 81.38 & 92.08 & 92.03 & 73.09 & 71.47 \\
LFM~\cite{NEURIPS2024_71b17f00} & 66.37 & 66.02 & 70.02 & 69.55 & 47.45 & 46.50 & 79.90 & 71.60 & 84.95 & 84.35 & 92.58 & 92.54 & 73.54 & 71.76 \\
Relearning~\cite{wei2024diagnosing} & 65.91 & 65.58 & 71.23 & 73.72 & 48.17 & 48.02 & 78.55 & 69.92 & 82.66 & 82.03 & 91.43 & 91.53 & 72.99 & 71.80 \\
ReconBoost~\cite{hua2024reconboost} & 68.68 & 67.92 & 74.57 & 74.27 & 48.26 & 47.33 & 81.01 & 74.13 & 82.81 & 82.13 & 92.21 & 92.15 & 74.59 & 72.99 \\
GGDM~\cite{hu2025ggdm} & \second{74.64} & 73.80 & \second{83.41} & \second{84.03} & 50.23 & 49.98 & -- & -- & 81.69 & 80.77 & 92.41 & 92.33 & -- & -- \\
MMPareto~\cite{wei2024mmpareto} & 74.55 & \second{74.21} & 79.97 & 80.57 & 49.19 & 48.37 & \second{81.21} & \second{74.66} & \second{85.12} & 84.68 & \second{92.85} & \second{92.73} & \second{77.15} & \second{75.87} \\
\rowcolor{gbrowblue}
GeoBalance & \best{78.64} & \best{77.76} & \best{87.33} & \best{86.63} & \best{52.19} & \best{51.38} & \best{82.21} & \best{76.22} & \best{85.67} & \best{84.94} & \best{93.13} & \best{92.97} & \best{79.86} & \best{78.32} \\
\bottomrule
\end{tabular*}
\end{table*}

\subsection{Experimental Setup}
\label{sec:evaluation_protocol}
\label{sec:data_protocol}
Experiments cover Kinetics-Sounds~\cite{arandjelovic2017look},
CREMA-D~\cite{cao2014crema}, and VGGSound~\cite{Chen20} (video--audio),
CMU-MOSEI~\cite{bagher-zadeh-etal-2018-multimodal} (text--video--audio),
UCF101~\cite{soomro2012ucf101dataset101human} (RGB--flow), and
Food-101~\cite{wang2015recipe} (image--text).
Backbones, preprocessing, and released partitions follow
BalanceBenchmark~\cite{xu2025balancebenchmark}.
Audio-visual datasets and UCF101 use modality-specific ResNet-18 branches;
Food-101 pairs ResNet-18 with BERT; CMU-MOSEI uses modality-specific
Transformers over released feature sequences.
Accuracy and macro-$F_1$ measure overall and class-balanced prediction.

GeoBalance initializes the modality branches from unimodal checkpoints
and optimizes the complete model end-to-end.
All comparisons use the benchmark data splits and backbone architectures.
Training uses AdamW for 100 epochs on four NVIDIA GeForce RTX 4090 GPUs.
The monitor uses $\eta=0.7$, $\rho=0.99$, and one warm-up epoch across datasets;
$\gamma=1$ and $\epsilon=10^{-8}$ are fixed.
Detailed preprocessing and baseline descriptions are provided in the
supplementary material.

\subsection{Comparison with Existing Methods}
\label{sec:main_results}
\paragraph{Overall comparison.}
\Cref{tab:main_results} shows that GeoBalance achieves the strongest
reported accuracy and macro-$F_1$ across the evaluated benchmarks.
The agreement between the two metrics suggests that the improvement
extends beyond overall correctness to class-balanced prediction.
Reconstruction reaches the weak encoder, allowing the class scaffold
and spectral penalty to reshape the features consumed by fusion.
Asymmetric projection then limits interference from the joint objective,
so that an auxiliary target can influence the representation even when
another modality supplies an easier route to prediction.

\paragraph{Cross-dataset behavior.}
The strongest competing method varies with the dataset and metric.
GGDM and MMPareto are competitive on the audio-visual emotion and event
tasks, while G-Blending remains strong on VGGSound and UCF101.
This variation suggests that allocating optimization effort and preserving
representation geometry address related but distinct bottlenecks.
GeoBalance's consistent ordering is compatible with a geometric criterion
that does not require equal modality contributions or comparable feature
scales across branches.

\begin{figure*}[t]
\centering
\includegraphics[width=\linewidth]{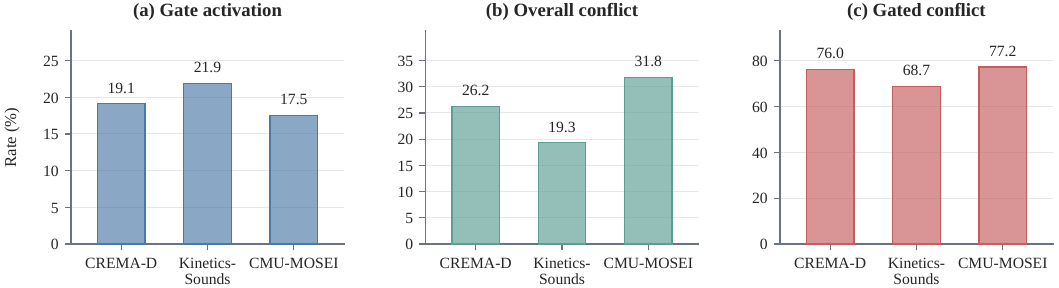}
\caption{Monitoring and gradient-conflict rates (\%). (a) Gate activation among post-warm-up iterations. (b) Conflict, $\langle\mathbf{g}_{J},\mathbf{g}_{G}\rangle<0$, among all post-warm-up iterations. (c) Conflict among gated iterations.}
\label{fig:monitor_conflict}
\end{figure*}

\subsection{Ablation Study}
\label{sec:component_analysis}

\begin{table}[t]
    \caption{Ablation study of GeoBalance. Each row removes one component while retaining the others.}
    \label{tab:ablation}
    \centering
    \small
    \setlength{\tabcolsep}{5pt}
    \renewcommand{\arraystretch}{1.08}
    \begin{tabular*}{\linewidth}{@{\extracolsep{\fill}}lcccc@{}}
        \toprule
        & \multicolumn{2}{c}{CREMA-D}
        & \multicolumn{2}{c}{Kinetics-Sounds} \\
        \cmidrule(lr){2-3}\cmidrule(lr){4-5}
        Configuration & Acc. & $F_1$ & Acc. & $F_1$ \\
        \midrule
        \rowcolor{gbrowblue}
        Full model & \best{87.33} & \best{86.63}
                   & \best{78.64} & \best{77.76} \\
        w/o monitor & 78.19 & 77.44 & 73.21 & 73.28 \\
        w/o ETF & 78.67 & 79.55 & 73.50 & 73.64 \\
        w/o spectral & 81.32 & 81.21 & 74.92 & 74.22 \\
        w/o projection & 79.12 & 79.43 & 74.71 & 73.91 \\
        \bottomrule
    \end{tabular*}
\end{table}

\paragraph{When to reconstruct.}
Removing the monitor yields the weakest results on both datasets in
\cref{tab:ablation}, even though both reconstruction losses remain.
This indicates that their timing matters alongside their availability.
Applying a fixed class scaffold throughout training can constrain early
co-adaptation, providing a
plausible explanation for the benefit of gating.
The ablation changes both the state criterion and auxiliary-loss exposure.

\paragraph{What to reconstruct.}
Removing ETF reconstruction is more detrimental than removing spectral
regularization on both datasets and metrics.
The class scaffold supplies an explicit discriminative target when fusion
can rely on another modality, whereas the spectral term discourages
concentrated variation without specifying which directions separate classes.
Their effects are complementary: preserving variation cannot replace
class guidance, while the decline without spectral regularization shows
that class guidance alone does not recover the full model's performance.
Both constraints therefore contribute to reconstruction.

\paragraph{How to preserve reconstruction.}
Removing projection also degrades performance despite retaining the monitor
and both reconstruction losses.
Its effect is comparable to removing ETF reconstruction on CREMA-D,
with their ordering differing between accuracy and macro-$F_1$;
on Kinetics-Sounds, ETF removal is more damaging.
These two ablations therefore show no consistent ordering across datasets
and metrics.
Conflicting joint updates can partially cancel a useful reconstruction
signal, explaining why a suitable target also benefits from gradient
protection. Together, these results support coordinating intervention
timing, geometric targets, and gradient~protection.

\subsection{Effect of Backbone Training Protocols}
\label{sec:backbone_protocols}

We further evaluate GeoBalance under three training
protocols to examine the effects of encoder trainability
and initialization. \emph{Frozen} fixes pretrained
backbones while updating modality-specific adapters
and fusion layers. \emph{Trainable} optimizes the entire
model from random initialization, whereas
\emph{Fine-tuned} initializes from pretrained unimodal
checkpoints and updates all parameters.

\begin{table}[t]
    \caption{Comparison of GeoBalance under three training protocols on three benchmarks (\%). Each entry reports accuracy / macro-F1.}
    \label{tab:backbone_protocols}
    \centering
    \small
    \setlength{\tabcolsep}{5pt}
    \renewcommand{\arraystretch}{1.08}
    \begin{tabular*}{\linewidth}{@{\extracolsep{\fill}}lccc@{}}
        \toprule
        Protocol & Kinetics-Sounds & CREMA-D & VGGSound \\
        \midrule
        Frozen
        & 72.30 / 71.63 & 77.67 / 78.12 & 49.91 / 49.79 \\
        Trainable
        & 77.73 / 76.91 & 83.27 / 83.14 & 51.66 / 51.02 \\
        \rowcolor{gbrowblue}
        Fine-tuned
        & \best{78.64} / \best{77.76}
        & \best{87.33} / \best{86.63}
        & \best{52.19} / \best{51.38} \\
        \bottomrule
    \end{tabular*}
\end{table}

As shown in \Cref{tab:backbone_protocols}, Fine-tuned
consistently achieves the strongest performance,
while Trainable also outperforms Frozen across all
three benchmarks. Comparing Frozen with Fine-tuned
suggests that updating pretrained encoders provides
greater flexibility for geometric reconstruction.
Meanwhile, the advantage of Fine-tuned over Trainable
indicates that pretrained initialization further
facilitates learning. These effects are particularly
pronounced on CREMA-D, where weak-modality
representations benefit from both trainable backbones
and informative initialization. Overall, the results
suggest that GeoBalance benefits from sufficient
representation adaptability, while remaining
applicable under different training protocols.

\subsection{Representation Geometry}
\label{sec:representation_analysis}

\begin{figure}[t]
    \centering
    \includegraphics[width=0.95\linewidth]{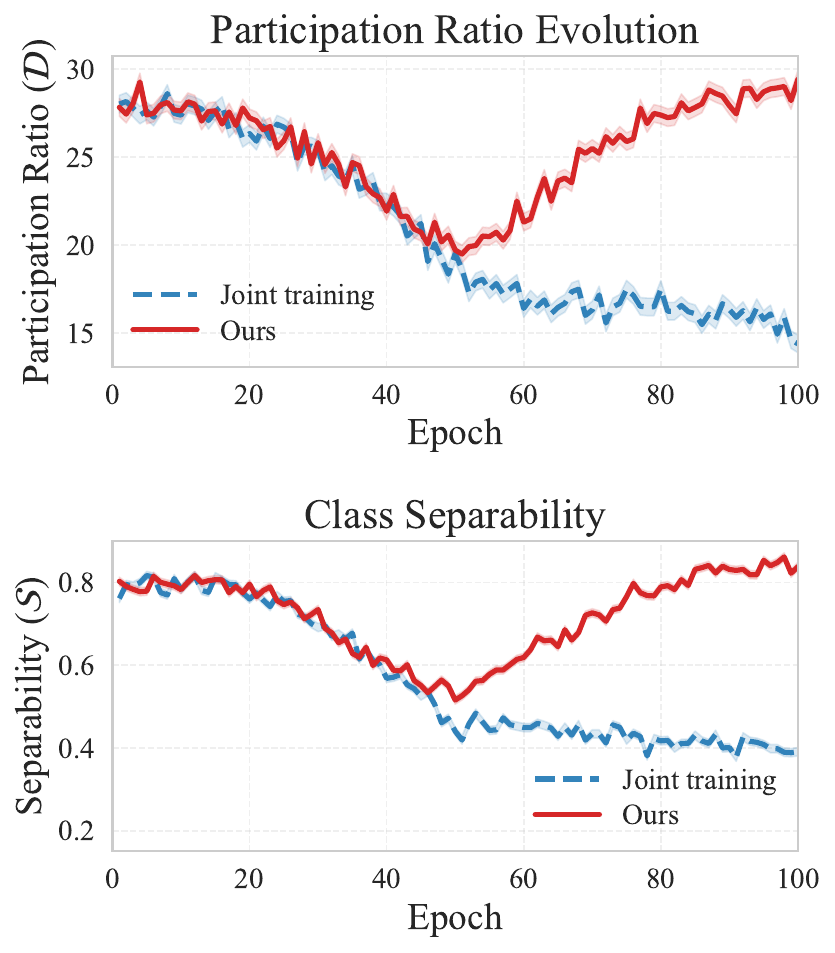}
    \caption{Weak-modality geometry during training on CREMA-D.
    Top: participation ratio. Bottom: class separability.}
    \label{fig:geom_curve}
\end{figure}

\begin{figure*}[!htbp]
\centering
\includegraphics[width=\linewidth]{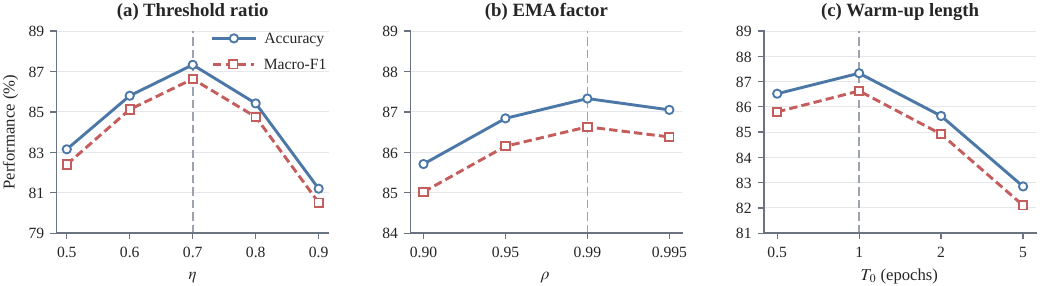}
\caption{Monitoring-parameter sensitivity on CREMA-D. One parameter is varied at a time: (a) threshold ratio $\eta$, (b) EMA factor $\rho$, and (c) warm-up length $T_0$. The remaining parameters are fixed.}
\label{fig:hyperparameter_sensitivity}
\end{figure*}

\Cref{fig:geom_curve} illustrates the evolution of weak-modality
representation geometry during training. Both joint training and
GeoBalance initially exhibit declining participation ratio and
class separability, but their subsequent trajectories diverge.
Joint training sustains this deterioration, whereas GeoBalance
recovers both measures. This coordinated recovery suggests that
GeoBalance restores class-conditional structure rather than
merely increasing feature diversity.

The two metrics capture complementary geometric properties.
The participation ratio measures how within-class variation
is distributed across feature directions, while class
separability reflects discrimination between classes.
A higher separability ratio alone may result from reduced
within-class scatter, even when the remaining variation
is concentrated along few directions. Their joint recovery
therefore indicates both richer within-class geometry and
improved class discrimination, consistent with the
complementary roles of ETF reconstruction and spectral
regularization. These epoch-wise diagnostics characterize
the learned geometry rather than the minibatch gate's
activation times.

\begin{figure}[t]
\centering
\includegraphics[width=\linewidth]{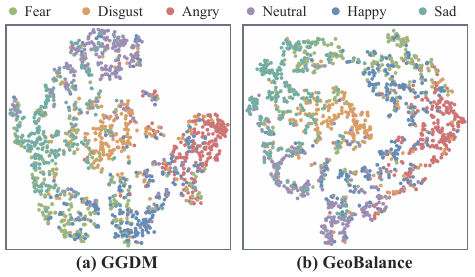}
\caption{Weak-modality $t$-SNE views on CREMA-D. Colors mark classes.}
\label{fig:tsne}
\end{figure}

The $t$-SNE views in \cref{fig:tsne} provide a complementary class-level
perspective.
Compared with GGDM, GeoBalance exhibits clearer regional class grouping.
This qualitative structure agrees with the geometric trends without
implying that all classes become separable in the original feature space.

\subsection{Monitoring and Gradient Conflict}
\label{sec:monitoring_behavior}

\Cref{fig:monitor_conflict} examines the relationship between
geometric monitoring and gradient conflict.
The monitor activates on only a minority of training
iterations, while gradient conflict occurs much more
frequently within the gated subset than across training
as a whole. This suggests that geometric deterioration
often coincides with opposing optimization directions.
However, the two conditions serve distinct purposes:
the monitor determines when reconstruction is needed,
whereas gradient projection protects reconstruction
only when the joint gradient conflicts with it.

The distinction becomes clearer across datasets.
Overall gradient conflict does not directly determine
gate activation frequency, indicating that gradient
interference alone is insufficient to characterize
representation deterioration. This explains why
GeoBalance combines geometric monitoring with asymmetric
projection rather than relying on either mechanism alone.
The monitor avoids imposing reconstruction on otherwise
usable representations, while projection prevents
conflicting joint updates from weakening the
reconstruction signal. When the gradients agree, the
joint update remains unchanged. Together, these
mechanisms enable selective intervention in both
training time and gradient direction.

\subsection{Hyperparameter Sensitivity}
\label{sec:hyperparameter_analysis}

\Cref{fig:hyperparameter_sensitivity} examines controls that determine
when reconstruction is applied.
The threshold ratio produces a clear rise-and-fall trend.
A conservative gate requires more deterioration before intervention;
a permissive gate can also react to modest fluctuations.
The observed interior optimum is consistent with a balance between
responsiveness and allowing ordinary co-adaptation.

The EMA factor has a smoother effect: stronger averaging helps over
much of the tested range, followed by a mild decline at the largest factor.
This fits the trade-off between stable estimates and response delay.
Warm-up has a different role because it determines both the reference
window and the earliest allowed intervention.
Short calibration works better than prolonged warm-up, consistent
with the cost of delaying reconstruction and allowing the reference
to reflect later training states.

%

\section{Conclusion}
\label{sec:conclusion}
GeoBalance addresses a representation bottleneck in multimodal learning: a weak branch can lose usable class structure even while the fused objective improves. Its monitor tracks the coupled degradation defining MMC, ETF and spectral objectives guide reconstruction, and asymmetric projection reduces gradient interference. Results across six benchmarks support preserving weak-modality geometry, especially when the branch can respond to reconstruction. Comparable gradient magnitudes can coexist with different class-conditional geometry, while unequal predictive strengths can preserve complementary evidence.


\FloatBarrier
\section*{LLM Usage Statement}
In this work, large language models ChatGPT and Gemini were used solely as an assistive tool to edit and refine author-produced content, including polishing prose and related work descriptions, debugging and refining code and experimental workflows, and clarifying author-originated ideas, analyses, and claims. All outputs and citations were manually verified by the authors, who assume full responsibility for the final submission.

{
  \small
  \bibliographystyle{ieeenat_fullname}
  \bibliography{main}
}
\clearpage
\appendix
\setcounter{figure}{0}
\setcounter{table}{0}
\setcounter{equation}{0}
\setcounter{algorithm}{0}
\renewcommand{\thefigure}{S\arabic{figure}}
\renewcommand{\thetable}{S\arabic{table}}
\renewcommand{\theequation}{S\arabic{equation}}
\renewcommand{\thealgorithm}{S\arabic{algorithm}}
\renewcommand{\theHfigure}{supp.\arabic{figure}}
\renewcommand{\theHtable}{supp.\arabic{table}}
\renewcommand{\theHequation}{supp.\arabic{equation}}
\maketitlesupplementary
\label{page:supp_start}
\nocite{lin2025ssmc}

\section{Notation and Geometric Properties}
\label{sec:supp_preliminaries}
\label{supp:sec:preliminaries}
This section develops the geometric quantities in
\cref{sec:preliminaries}. \Cref{supp:tab:notation} summarizes the notation.

\begin{table*}[t]
\caption{Core notation. Superscript $(m)$ denotes modality $m$ and $\star$ the fixed weak modality. Hats denote minibatch estimates; overbars on $D$ and $\mathcal S$ denote EMAs.}
\label{supp:tab:notation}
\centering
\footnotesize
\setlength{\tabcolsep}{4.6pt}
\renewcommand{\arraystretch}{1.10}
\begin{tabular}{@{}p{2.0cm}p{5.55cm}p{2.0cm}p{5.55cm}@{}}
\toprule
Symbol & Meaning & Symbol & Meaning \\
\midrule
$N,M,K$ & Numbers of samples, modalities, and classes. &
$i,m,k$ & Sample, modality, and class indices. \\
$x_i^{(m)},y_i$ & Modality input and class label. &
$f_m,d_m$ & Encoder and its feature dimension. \\
$\phi,h,\mathrm{CE}$ & Fusion module, classifier, and cross entropy. &
$m_\star,d_\star$ & Fixed weak branch and its feature dimension. \\
$\mathbf z_i^\star$ & Weak representation used by joint prediction. &
$\boldsymbol\Sigma_{\mathrm w}^{(m)},\boldsymbol\Sigma_{\mathrm b}^{(m)}$ & Population within- and between-class scatter. \\
$D^{(m)},\mathcal S^{(m)}$ & Effective dimension and class separability. &
$\epsilon$ & Positive numerical stabilizer. \\
$\mathcal B_t,B$ & Minibatch index set and its size at iteration $t$. &
$\mathcal Y_t,\mathcal B_{t,k},n_{k,t}$ & Observed classes, class-$k$ index set, and its size. \\
$\widehat{\boldsymbol\mu}_{k,t}^\star,\widehat{\boldsymbol\mu}_t^\star$ & Class and overall minibatch feature means. &
$\widetilde{\mathbf Z}_t^\star,\mathbf G_t^\star$ & Detached class-centered features and raw Gram matrix. \\
$\widehat D_t^\star,\widehat{\mathcal S}_t^\star$ & Minibatch geometric estimates. &
$\overline D_t^\star,\overline{\mathcal S}_t^\star$ & Their EMA-smoothed values. \\
$\rho,T_0,\eta$ & EMA factor, warm-up length, and threshold ratio. &
$D^{\mathrm{ref}},\mathcal S^{\mathrm{ref}}$ & Fixed warm-up reference statistics. \\
$\tau_D,\tau_{\mathcal S}$ & Reference-scaled thresholds. &
$\mathbb I_t$ & Gate requiring both statistics below threshold. \\
$\mathbf C,\mathbf c_k$ & Fixed ETF matrix and class-$k$ prototype. &
$d_e,g_\star,\mathbf u_i^\star$ & Auxiliary dimension, head, and projected feature. \\
$s_\star,\gamma$ & Positive logit scale and standard-deviation floor. &
$\boldsymbol\mu_{u,t}^\star,\mathbf U_{c,t}^\star$ & Projected minibatch mean and globally centered features. \\
$\mathbf Q_t^\star,Q_{t,pq}^\star$ & Sample covariance and its scalar $(p,q)$ entry. &
$\sigma_{j,t}$ & Standard deviation of auxiliary coordinate $j$. \\
$\mathcal L_{\mathrm{joint}}$ & Fused prediction loss. &
$\mathcal L_{\mathrm{ETF}}^\star,\mathcal L_{\mathrm{spec}}^\star$ & Class-scaffold and variance--covariance losses. \\
$\mathcal L_{\mathrm{geo}}^\star$ & Sum of the two reconstruction losses. &
$\theta_{\mathrm{sh}}$ & Trainable parameters used by both objectives. \\
$\mathbf g_J,\mathbf g_G,\delta_t$ & Joint gradient, geometry gradient, and their inner product. &
$\Pi_t,\mathbf g_{\mathrm{sh}}$ & Stabilized joint-gradient correction and final shared gradient. \\
\bottomrule
\end{tabular}
\end{table*}

\subsection{Population Geometry}
\label{supp:sec:geometric_health}
\label{supp:sec:problem_setup}
For modality $m$, let $\pi_k=\Pr(y=k)$,
$\boldsymbol\mu_k^{(m)}=\mathbb E[\mathbf z^{(m)}\mid y=k]$, and
$\boldsymbol\mu^{(m)}=\mathbb E[\mathbf z^{(m)}]$.
The traces of the within- and between-class scatter matrices satisfy
\begin{equation}
\begin{aligned}
\operatorname{tr}(\boldsymbol\Sigma_{\mathrm w}^{(m)})
 &=\mathbb E\|\mathbf z^{(m)}-\boldsymbol\mu_y^{(m)}\|_2^2,\\
\operatorname{tr}(\boldsymbol\Sigma_{\mathrm b}^{(m)})
 &=\sum_{k=1}^K\pi_k
   \|\boldsymbol\mu_k^{(m)}-\boldsymbol\mu^{(m)}\|_2^2.
\end{aligned}
\label{supp:eq:scatter_traces}
\end{equation}
Thus $\mathcal S^{(m)}$ relates the spread of class means to average
within-class dispersion. The law of total covariance gives
\begin{equation}
\operatorname{Cov}(\mathbf z^{(m)})
=\boldsymbol\Sigma_{\mathrm w}^{(m)}
 +\boldsymbol\Sigma_{\mathrm b}^{(m)}.
\label{supp:eq:total_covariance}
\end{equation}
Global feature variation consequently contains contributions from both
within-class variation and between-class separation.

For a nonzero positive semidefinite matrix of rank $r$, let
$\lambda_1,\ldots,\lambda_r>0$ be its nonzero eigenvalues.
Nonnegativity and Cauchy--Schwarz yield
\begin{equation}
\sum_{j=1}^r\lambda_j^2
\le\left(\sum_{j=1}^r\lambda_j\right)^2
\le r\sum_{j=1}^r\lambda_j^2.
\label{supp:eq:participation_bounds}
\end{equation}
The unstabilized participation ratio therefore lies between $1$ and $r$.
It attains $r$ for equal nonzero eigenvalues and is unchanged by
multiplication of the matrix by a positive scalar. This property
distinguishes directional concentration from uniform contraction.

\subsection{Complementarity of the Two Statistics}
An analytical example illustrates why dimensional compression and class
separation should be considered jointly. Consider balanced two-class
Gaussian representations with means
$\boldsymbol\mu_\pm=\pm\sqrt b\,\mathbf e_1$, where $b>0$,
and common within-class covariance
$\operatorname{diag}(1,a)$, where $0<a\le1$.
In the unstabilized regime ($\epsilon = 0$),
\begin{equation}
D(a)=\frac{(1+a)^2}{1+a^2},
\qquad
\mathcal S(a)=\frac{b}{1+a}.
\label{supp:eq:compression_example}
\end{equation}
As $a$ decreases toward zero, $D(a)$ decreases from $2$ toward $1$,
whereas $\mathcal S(a)$ increases from $b/2$ toward $b$.
The class-discriminating coordinate remains $\mathbf e_1$;
compression removes variation in the orthogonal coordinate.
This change does not satisfy the coupled-decrease definition of MMC.

Conversely, scaling both class means by $c\in(0,1)$ while keeping their
within-class covariance fixed leaves $D$ unchanged and multiplies
$\mathcal S$ by $c^2$.
The two constructions isolate changes in directional concentration and
mean separation, respectively. They explain the distinct information
contributed by the two monitoring statistics.

\subsection{An ETF Construction}
\label{supp:sec:etf_prelim}
Let $\mathbf V\in\mathbb R^{K\times(K-1)}$ have orthonormal columns
spanning the subspace orthogonal to $\mathbf1_K$.
Let $\mathbf E\in\mathbb R^{d_e\times(K-1)}$ be an isometric embedding,
with $\mathbf E^\top\mathbf E=\mathbf I_{K-1}$.
The matrix
\begin{equation}
\mathbf C=\sqrt{\frac K{K-1}}\,\mathbf E\mathbf V^\top
\label{supp:eq:etf_construction}
\end{equation}
has the Gram matrix
\begin{equation}
\begin{split}
\mathbf C^\top\mathbf C
&=\frac K{K-1}\mathbf V\mathbf V^\top\\
&=\frac K{K-1}
 \left(\mathbf I_K-\frac1K\mathbf1_K\mathbf1_K^\top\right).
\end{split}
\label{supp:eq:etf_gram_derivation}
\end{equation}
Each prototype has unit norm, and distinct prototypes have inner product
$-1/(K-1)$. Their squared pairwise distance is $2K/(K-1)$.
The construction uses a $(K-1)$-dimensional span embedded in
$\mathbb R^{d_e}$, making the requirement $d_e\ge K-1$ explicit.
An orthogonal change of orientation preserves these relations.

\section{Method Details}
\label{sec:supp_method}

\subsection{Minibatch Geometry}
\label{supp:sec:batch_geometry}
\label{supp:sec:online_monitor}
The monitor uses detached weak-modality features.
For class $k$ observed in $\mathcal B_t$, the empirical means are
\begin{equation}
\begin{aligned}
\widehat{\boldsymbol\mu}_{k,t}^\star
 &=\frac1{n_{k,t}}\sum_{i\in\mathcal B_{t,k}}\mathbf z_i^\star,\\
\widehat{\boldsymbol\mu}_t^\star
 &=\frac1B\sum_{i\in\mathcal B_t}\mathbf z_i^\star.
\end{aligned}
\label{supp:eq:batch_means}
\end{equation}
Using the class-centered matrix $\widetilde{\mathbf Z}_t^\star$ from
\cref{sec:online_monitor}, the empirical scatter matrices are
\begin{equation}
\widehat{\boldsymbol\Sigma}_{\mathrm w,t}^\star
=\frac1B\widetilde{\mathbf Z}_t^{\star\top}\widetilde{\mathbf Z}_t^\star,
\label{supp:eq:empirical_within}
\end{equation}
\begin{equation}
\begin{split}
\widehat{\boldsymbol\Sigma}_{\mathrm b,t}^\star
={}&\frac1B\sum_{k\in\mathcal Y_t}n_{k,t}\\
&\cdot(\widehat{\boldsymbol\mu}_{k,t}^\star-\widehat{\boldsymbol\mu}_t^\star)
(\widehat{\boldsymbol\mu}_{k,t}^\star-\widehat{\boldsymbol\mu}_t^\star)^\top.
\end{split}
\label{supp:eq:empirical_between}
\end{equation}
The factors $n_{k,t}/B$ retain the empirical class proportions.
Taking traces and multiplying the resulting ratio by $B/B$ gives
\cref{eq:batch_separability}, including its $B\epsilon$ stabilizer.

The nonzero eigenvalues of the raw Gram matrix
$\mathbf G_t^\star$ equal those of
$B\widehat{\boldsymbol\Sigma}_{\mathrm w,t}^\star$.
Consequently,
\begin{equation}
\widehat D_t^\star
=\frac{\operatorname{tr}(\widehat{\boldsymbol\Sigma}_{\mathrm w,t}^\star)^2}
 {\|\widehat{\boldsymbol\Sigma}_{\mathrm w,t}^\star\|_F^2+\epsilon/B^2}.
\label{supp:eq:gram_normalization}
\end{equation}
This identity relates the minibatch computation to the population
definition and permits evaluation through traces and Frobenius norms
without eigendecomposition.

Each EMA is initialized with its first estimate and then updated at
every iteration. The median of its warm-up values defines the fixed
reference in \cref{eq:warmup_reference}.
For positive references, the gate can equivalently be expressed through
dimensionless relative statistics:
\begin{equation}
r_{D,t}=\frac{\overline D_t^\star}{D^{\mathrm{ref}}},
\qquad
r_{\mathcal S,t}=
\frac{\overline{\mathcal S}_t^\star}{\mathcal S^{\mathrm{ref}}}.
\label{supp:eq:relative_statistics}
\end{equation}
After warm-up, $\mathbb I_t=1$ precisely when
$r_{D,t}<\eta$ and $r_{\mathcal S,t}<\eta$.
The references remain fixed while the EMAs continue to evolve, so
intervention follows deterioration and recovery relative to the same
early-training state.

\subsection{Reconstruction and Gradient Flow}
\label{supp:sec:weak_reconstruction}
Monitoring and reconstruction use two connected feature spaces.
The weak representation $\mathbf z_i^\star$ is consumed by fusion, and
the training-only head maps it to
$\mathbf u_i^\star=g_\star(\mathbf z_i^\star)$.
The ETF term supplies class-dependent directions in the auxiliary space;
the variance--covariance term regularizes its global variation.

\begin{table}[t]
\caption{Feature spaces used by the monitor and reconstruction.
The auxiliary space is connected to the fused representation through
the training head.}
\label{supp:tab:feature_spaces}
\centering
\small
\setlength{\tabcolsep}{3pt}
\renewcommand{\arraystretch}{1.12}
\begin{tabular}{@{}>{\raggedright\arraybackslash}p{0.25\linewidth}>{\raggedright\arraybackslash}p{0.16\linewidth}>{\raggedright\arraybackslash}p{0.51\linewidth}@{}}
\toprule
Quantity & Feature & Construction \\
\midrule
$D,\mathcal S$ & $\mathbf z^\star$ &
Class-centered scatter and class means; detached. \\
$\mathcal L_{\mathrm{ETF}}^\star$ & $\mathbf u^\star$ &
Normalized features scored against fixed class prototypes. \\
$\mathcal L_{\mathrm{spec}}^\star$ & $\mathbf u^\star$ &
Globally centered covariance and coordinate variances. \\
\bottomrule
\end{tabular}
\end{table}

The same total-covariance decomposition as
\cref{supp:eq:total_covariance} applies in the auxiliary space:
\begin{equation}
\operatorname{Cov}(\mathbf u^\star)
=\boldsymbol\Sigma_{\mathrm w,u}^\star
 +\boldsymbol\Sigma_{\mathrm b,u}^\star.
\label{supp:eq:aux_total_covariance}
\end{equation}
Thus the spectral objective acts on variation contributed by both
within-class and between-class structure. Its combination with the ETF
term couples a class target with coordinate-level regularization.

For parameters $\theta_{\mathrm{sh}}$ shared by the two objectives,
the reconstruction gradient follows the chain rule
\begin{equation}
\nabla_{\theta_{\mathrm{sh}}}\mathcal L_{\mathrm{geo}}^\star
=\sum_{i\in\mathcal B_t}
 \mathbf J_{z_i,\theta_{\mathrm{sh}}}^{\top}
 \mathbf J_{g_\star}(\mathbf z_i^\star)^{\top}
 \nabla_{\mathbf u_i^\star}\mathcal L_{\mathrm{geo}}^\star,
\label{supp:eq:reconstruction_chain}
\end{equation}
where $\mathbf J$ denotes the corresponding Jacobian.
This path allows supervision in the auxiliary space to update the
features used by fusion. Detaching the statistics for monitoring
does not detach this reconstruction path.

\subsection{Asymmetric Projection}
\label{supp:sec:gradient_projection}
Both gradients in \cref{eq:two_gradients} are evaluated on the same
trainable parameter vector $\theta_{\mathrm{sh}}$.
Let $\Pi_t^0$ denote the correction with $\epsilon=0$ for nonzero
$\mathbf g_G$, and define it as the identity when $\mathbf g_G=0$.

\begin{proposition}
\label{supp:prop:minimal_projection}
When $\mathbb I_t=1$, $\Pi_t^0(\mathbf g_J;\mathbf g_G)$ is the unique
solution of
\begin{equation}
\begin{aligned}
&\min_{\mathbf v}\ \frac12\|\mathbf v-\mathbf g_J\|_2^2\\
&\text{subject to }\ \langle\mathbf v,\mathbf g_G\rangle\ge0.
\end{aligned}
\label{supp:eq:projection_program}
\end{equation}
\end{proposition}
\begin{proof}
If $\mathbf g_G=0$, every vector is feasible and $\mathbf g_J$ uniquely
minimizes the objective. If $\delta_t\ge0$, $\mathbf g_J$ is already
feasible and again achieves the minimum.
Otherwise, the Lagrangian stationarity condition gives
$\mathbf v=\mathbf g_J+\alpha\mathbf g_G$ with $\alpha\ge0$.
The active boundary requires
$\alpha=-\delta_t/\|\mathbf g_G\|_2^2$, yielding the ideal correction.
Strict convexity establishes uniqueness.
\end{proof}

The correction retains the component orthogonal to $\mathbf g_G$
and removes the opposing parallel component. After adding the active
reconstruction gradient, the ideal shared update satisfies
\begin{equation}
\langle\mathbf g_{\mathrm{sh}}^0,\mathbf g_G\rangle
=\|\mathbf g_G\|_2^2+\max\{\delta_t,0\}.
\label{supp:eq:ideal_shared_alignment}
\end{equation}
The reconstruction direction therefore has a fixed priority in this
local gradient construction.
The underlying half-space projection is used in conflict-gradient
methods such as PCGrad~\cite{yu2020gradient}; GeoBalance conditions it on
the geometric state and applies the correction to the joint gradient
alone.

\subsection{Optimization Procedure}
\label{supp:sec:optimization}
The parameter partition follows computational dependence on the two
objectives. Parameters used by both objectives receive
\cref{eq:shared_gradient}; joint-only parameters receive
$\nabla\mathcal L_{\mathrm{joint}}$; geometry-only parameters receive the
gated reconstruction gradient.
Projection is computed on $\theta_{\mathrm{sh}}$ as one vector, before
the optimizer updates its moments.

In the end-to-end setting, the shared set includes the trainable weak
encoder, while the other encoders and fusion classifier retain their
joint-objective gradients. The auxiliary head and its logit scale are
geometry-only parameters. Under a frozen feature extractor, any change
to the monitored representation must pass through a trainable adapter
upstream of that representation.
This dependency specifies where geometric supervision can act in each
training protocol.

Algorithm~\ref{alg:geobalance} gives the complete procedure.
An inactive gate retains ordinary joint training while detached
monitoring continues.

\begin{algorithm*}[t]
\caption{Monitor--reconstruct--project training of GeoBalance}
\label{alg:geobalance}
\begin{algorithmic}[1]
\REQUIRE Labeled dataset $\mathcal D$; encoders $\{f_m\}_{m=1}^M$; fusion $\phi$; classifier $h$; fixed weak index $m_\star$; auxiliary head $g_\star$; ETF matrix $\mathbf C$; $T_0,\rho,\eta$; optimizer.
\STATE Initialize auxiliary parameters and positive scale $s_\star$; identify $\theta_{\mathrm{sh}}$ by dependence on both objectives.
\FOR{training iteration $t=1,2,\ldots$}
    \STATE Sample $\mathcal B_t$, compute modality features and $\mathcal L_{\mathrm{joint}}$.
    \STATE \textbf{Monitor:} compute \cref{eq:batch_dimension,eq:batch_separability} on detached weak features.
    \STATE Initialize EMAs from the first estimates, or update them using \cref{eq:ema_monitor}.
    \IF{$t\le T_0$}
        \STATE Set $\mathbb I_t=0$; at $t=T_0$, store the reference medians from \cref{eq:warmup_reference}.
    \ELSE
        \STATE Evaluate the collapse gate using \cref{eq:collapse_gate}.
    \ENDIF
    \IF{$\mathbb I_t=0$}
        \STATE Perform the standard joint-objective update.
    \ELSE
        \STATE \textbf{Reconstruct:} compute $\mathcal L_{\mathrm{geo}}^\star$ using \cref{eq:etf_loss,eq:spectral_loss,eq:geometry_loss}.
        \STATE \textbf{Project:} compute $\mathbf g_J,\mathbf g_G$ on $\theta_{\mathrm{sh}}$ and apply \cref{eq:projection_operator,eq:shared_gradient}.
        \STATE Assign the shared, joint-only, and geometry-only gradients to their parameter groups.
        \STATE Update optimizer moments and parameters using these assigned gradients.
    \ENDIF
\ENDFOR
\end{algorithmic}
\end{algorithm*}

\subsection{Computational Complexity}
\label{supp:sec:complexity}
The raw Gram matrix requires $O(B^2d_\star)$ operations and $O(B^2)$
storage at every iteration.
On active steps, ETF scoring requires $O(BKd_e)$ operations and
auxiliary covariance computation requires $O(Bd_e^2)$ operations and
$O(d_e^2)$ storage.
Gradient correction is linear in the number of shared parameters.

Let $p$ denote the active-step fraction. The average added computation
can be decomposed as
\begin{equation}
C_{\mathrm{extra}}=C_{\mathrm{monitor}}+pC_{\mathrm{active}},
\label{supp:eq:cost_accounting}
\end{equation}
where $C_{\mathrm{active}}$ includes head evaluation, ETF scoring,
covariance computation, additional differentiation, and projection.
The monitoring term remains present on inactive steps.
After training, the auxiliary head, prototypes, monitor, and gate states
are removed; inference uses the original deployed architecture.

\section{Experimental Protocols and Analysis}
\label{sec:supp_experiments}
This section details the experimental settings and develops the
interpretation of the comparisons, ablations, and diagnostics reported
in \cref{sec:experiments}.

\subsection{Datasets and Preprocessing}
\label{sec:supp_evaluation_protocol}
\begin{table}[t]
\caption{Benchmark modalities and backbone architectures.}
\label{tab:datasets}
\centering
\small
\setlength{\tabcolsep}{3pt}
\renewcommand{\arraystretch}{1.12}
\begin{tabular}{@{}>{\raggedright\arraybackslash}p{0.27\linewidth}>{\raggedright\arraybackslash}p{0.27\linewidth}>{\raggedright\arraybackslash}p{0.38\linewidth}@{}}
\toprule
Benchmark & Modalities & Backbones \\
\midrule
Kinetics-Sounds & Video, audio & ResNet-18 branches \\
CREMA-D & Video, audio & ResNet-18 branches \\
VGGSound & Video, audio & ResNet-18 branches \\
CMU-MOSEI & Text, video, audio & Modality-specific Transformers \\
UCF101 & RGB, optical flow & ResNet-18 branches \\
Food-101 & Image, text & ResNet-18, BERT \\
\bottomrule
\end{tabular}
\end{table}

Kinetics-Sounds~\cite{arandjelovic2017look} contains 31 audio-visual
event classes. CREMA-D~\cite{cao2014crema} contains recordings of
91 actors portraying six emotion categories.
VGGSound~\cite{Chen20} covers in-the-wild audio-visual events.
CMU-MOSEI~\cite{bagher-zadeh-etal-2018-multimodal} supplies text,
visual, and acoustic streams for sentiment analysis.
UCF101~\cite{soomro2012ucf101dataset101human} uses RGB and optical-flow
streams for action recognition, while Food-101~\cite{wang2015recipe}
provides the image-text setting.
The datasets span several modality combinations and classification tasks.

Backbones, preprocessing, and released partitions follow
BalanceBenchmark~\cite{xu2025balancebenchmark}.
The audio-visual datasets use modality-specific ResNet-18
branches~\cite{he2016resnet} for RGB frames and one-channel audio
time-frequency maps. UCF101 uses separate ResNet-18 branches for
RGB and two-channel flow.
Food-101 combines ResNet-18 with BERT~\cite{devlin2019bert}.
CMU-MOSEI uses modality-specific Transformer encoders over released
unaligned feature sequences.

\paragraph{Audio-visual inputs.}
Videos are sampled at 16 frames per second and resized to
$224\times224$.
Kinetics-Sounds and VGGSound select three temporally distributed
frames per clip, and CREMA-D selects two.
Training samples a random frame from each temporal segment and applies
random resized cropping and horizontal flipping.
Evaluation uses segment midpoints and deterministic resizing.
RGB inputs are normalized using ImageNet statistics.
Audio is resampled to 16~kHz and converted to log-Mel filterbanks
with a 25-ms window and a 10-ms hop.

\paragraph{Other modality combinations.}
UCF101 samples three temporally matched RGB and flow observations;
each flow input contains horizontal and vertical components.
Food-101 images are resized to $224\times224$, and BERT-tokenized text
is padded or truncated to 40 tokens.
CMU-MOSEI uses the standard unaligned sequences with text, acoustic,
and visual feature dimensions of 300, 74, and 35, respectively.

\subsection{Training Configuration}
GeoBalance selects the weak branch after separate unimodal training
under the same setting. The branch with the lowest training performance
is designated $m_\star$, as defined in \cref{sec:preliminaries}, and this
selection remains fixed during joint training.
The principal configuration initializes the modality branches from
unimodal checkpoints and updates the complete model end-to-end.
The auxiliary head receives geometric supervision on active steps.

\begin{table}[t]
\caption{Shared GeoBalance settings in the reported experiments.}
\label{supp:tab:settings}
\centering
\small
\setlength{\tabcolsep}{3pt}
\renewcommand{\arraystretch}{1.10}
\begin{tabular}{@{}>{\raggedright\arraybackslash}p{0.49\linewidth}>{\raggedright\arraybackslash}p{0.45\linewidth}@{}}
\toprule
Setting & Value \\
\midrule
Optimizer; training length & AdamW; 100 epochs \\
Hardware & Four RTX 4090 GPUs \\
Threshold ratio $\eta$ & $0.7$ \\
EMA factor $\rho$ & $0.99$ \\
Warm-up length $T_0$ & One epoch \\
Standard-deviation floor $\gamma$ & $1$ \\
Numerical stabilizer $\epsilon$ & $10^{-8}$ \\
Auxiliary dimension $d_e$ & $\max(d_\star,K-1)$ \\
Logit scale $s_\star$ & Positive, learned; initialized to $10$ \\
ETF and spectral coefficients & $1$ each \\
\bottomrule
\end{tabular}
\end{table}

The warm-up duration is expressed in epochs in
\cref{supp:tab:settings}; $T_0$ in the algorithm denotes the corresponding
number of training iterations.
The monitor parameters are shared across datasets.
The sensitivity study varies $\eta$, $\rho$, and the warm-up length
one at a time, retaining the remaining settings.

\begin{table}[t]
\caption{GeoBalance protocols evaluated in
\cref{tab:backbone_protocols}.}
\label{supp:tab:training_protocols}
\centering
\small
\setlength{\tabcolsep}{3pt}
\renewcommand{\arraystretch}{1.12}
\begin{tabular}{@{}>{\raggedright\arraybackslash}p{0.19\linewidth}>{\raggedright\arraybackslash}p{0.30\linewidth}>{\raggedright\arraybackslash}p{0.43\linewidth}@{}}
\toprule
Protocol & Initialization & Optimization scope \\
\midrule
Frozen & Unimodal checkpoints &
Frozen backbones; trainable adapters and prediction modules. \\
Trainable & Random &
End-to-end optimization of the model. \\
Fine-tuned & Unimodal checkpoints &
End-to-end optimization of the model. \\
\bottomrule
\end{tabular}
\end{table}

Frozen and Fine-tuned use the same unimodal checkpoint initialization.
Their comparison examines the role of encoder trainability.
Trainable and Fine-tuned both update the model end-to-end and differ in
initialization.
The shared-gradient parameter set follows the dependency rule in
\cref{supp:sec:optimization}; parameters outside the reconstruction path
receive only their applicable objective's gradient.

\subsection{Baseline Methods and Comparison Protocol}
The baselines address imbalance through several mechanisms.
G-Blending~\cite{wang2020what} adjusts objective weights using
generalization behavior; Greedy~\cite{wu2022characterizing} counteracts
optimization toward the easiest modality; and OGM-GE~\cite{peng2022balanced}
modulates modality gradients with gradient enhancement.
LFM~\cite{NEURIPS2024_71b17f00} models the modality gap during training.
MMCosine~\cite{xu2023mmcosine} uses cosine-based classifier interactions,
and PMR~\cite{fan2023pmr} supplies prototype-based supervision.
Relearning~\cite{wei2024diagnosing} introduces targeted relearning,
while ReconBoost~\cite{hua2024reconboost} coordinates unimodal learners
through boosting.
MMPareto~\cite{wei2024mmpareto} uses unimodal assistance in multi-objective
optimization, and GGDM~\cite{hu2025ggdm} modulates optimization using
geometric gradient divergence.

Accuracy and macro-$F_1$ measure overall and class-balanced prediction,
respectively. Macro-$F_1$ weights classes equally.
The ablation and monitoring-sensitivity experiments retain the
initialization and optimization scope of the principal GeoBalance
configuration.

\subsection{Analysis of the Reported Experiments}
\label{sec:supp_results}
\label{sec:supp_component_analysis}
\label{sec:supp_representation_analysis}

\paragraph{Component contributions.}
The ablations in \cref{tab:ablation} show lower accuracy and macro-$F_1$
on both CREMA-D and Kinetics-Sounds after removing each component.
The ETF term supplies class-dependent directions, while spectral
regularization shapes auxiliary feature variation.
Projection coordinates reconstruction with the joint objective.
Together, these comparisons support the contribution of each component
within the evaluated configuration.

\paragraph{Training protocols.}
Fine-tuned performs best among the three GeoBalance protocols on the
three datasets in \cref{tab:backbone_protocols}, followed by Trainable
and Frozen. The ordering is consistent with benefits from both
pretrained initialization and trainable representation capacity.
Frozen retains adaptability in downstream modules; end-to-end protocols
extend reconstruction to the encoder through the gradient path in
\cref{supp:eq:reconstruction_chain}.

\paragraph{Representation geometry.}
The weak-modality curves in \cref{fig:geom_curve} show an initial decline
in both participation ratio and class separability.
Their later trajectories diverge: joint training maintains lower
values, whereas GeoBalance recovers both quantities.
The participation ratio characterizes directional concentration of
within-class variation, and separability describes the relative spread
of class means.
Their coordinated recovery provides a complementary description of
the learned representation.
The $t$-SNE plots in \cref{fig:tsne} further illustrate regional
class grouping.

\paragraph{Monitoring and gradient conflict.}
Let $\mathcal T$ be the set of post-warm-up iterations and
$C_t=\mathbf1[\delta_t<0]$ the conflict indicator.
The three rates in \cref{fig:monitor_conflict} correspond to
\begin{equation}
\begin{aligned}
p_G&=\frac1{|\mathcal T|}\sum_{t\in\mathcal T}\mathbb I_t,\\
p_C&=\frac1{|\mathcal T|}\sum_{t\in\mathcal T}C_t,\\
p_{C\mid G}&=
\frac{\sum_{t\in\mathcal T}\mathbb I_t C_t}
     {\sum_{t\in\mathcal T}\mathbb I_t}.
\end{aligned}
\label{supp:eq:diagnostic_rates}
\end{equation}
The conditional rate is defined when at least one iteration is gated.
The gate is active on a minority of iterations, and conflict is more
frequent within that subset than across training as a whole.
These observations distinguish geometric state selection from the
gradient-direction test: the former activates reconstruction, and the
latter determines whether its joint gradient requires correction.

\paragraph{Monitoring controls.}
The sensitivity results in \cref{fig:hyperparameter_sensitivity}
characterize three aspects of the monitor.
The threshold ratio controls the deterioration required for activation;
the EMA factor controls smoothing; and warm-up controls calibration
and the earliest intervention time.
The observed trends are consistent with a trade-off between
responsiveness and stable state estimation.

\section{Discussion}
\label{sec:supp_discussion}

GeoBalance aims to preserve class-relevant structure in the weak branch
so that its complementary information remains accessible to fusion.
This objective accommodates unequal modality strengths.
Such asymmetry is consistent with settings in which prioritizing a
performance-dominant modality improves prediction~\cite{wei2026pdmp}.
The geometric criterion evaluates each weak branch relative to its own
reference state, and the reconstruction target is independent of the
dominant branch's features.

The training-protocol comparison highlights representation
adaptability: updating an encoder provides a broader parameter space
for reconstruction than updating only downstream adapters.
At inference, the training mechanisms are removed and prediction
follows the original fusion architecture.

\end{document}